\documentclass[a4paper,11pt]{article}

\usepackage{relsize}
\usepackage{epigraph}
\usepackage{nicefrac}
\usepackage[margin=1in]{geometry}
\usepackage{fullpage}
\usepackage[T1]{fontenc}

\usepackage{hyperref}
\usepackage{color}
\hypersetup{colorlinks=true,linkcolor=blue,citecolor=blue}
\usepackage{float}
\floatstyle{boxed}
\newfloat{algorithm}{t}{lop}
\usepackage{tikz}
\usepackage{varwidth}

\usepackage{mathtools}
\usepackage{float}
\usepackage[titletoc,title]{appendix}
\usepackage[classfont=sanserif,langfont=roman,funcfont=italic]{complexity}
\usepackage{caption}
\usepackage{comment}
\usepackage{enumerate}
\usepackage{amsmath, amsthm, amssymb, amstext,  graphicx,amsopn} 
\usepackage{subfigure}
\usepackage{tcolorbox}

\newtheorem{definition}{Definition}
\newtheorem{proposition}{Proposition}
\newtheorem{example}{Example}

\newtheorem{corollary}{Corollary}

\newtheorem{theorem}{Theorem}
\newtheorem{lemma}{Lemma}

\usepackage[T1]{fontenc}
\usepackage{enumitem}

\title{Online Learning with an Almost Perfect Expert}

\author{
	Simina Br\^anzei\footnote{Purdue University, USA. E-mail: \textcolor{blue}{\href{mailto:simina.branzei@gmail.com}{simina.branzei@gmail.com}}. Work done in part while visiting Microsoft Research.}\\
	\newline
	\and
	Yuval Peres\footnote{Microsoft Research, USA. E-mail: 
		\textcolor{blue}{\href{mailto:peres@microsoft.com}{peres@microsoft.com}}.}
}
\date{}
\begin{document}
\maketitle

\begin{abstract}
We study the multiclass online learning problem where a forecaster makes a sequence of predictions using the advice of $n$ experts. Our main contribution is to analyze the regime where the best expert makes at most $b$ mistakes and to show that when $b = o(\log_4{n})$, the expected number of mistakes made by the optimal forecaster is at most $\log_4{n} + o(\log_4{n})$. We also describe an adversary strategy showing that this bound is tight and that the worst case is attained for binary prediction.
\end{abstract}

\newpage

\section{Introduction}

We study the multiclass online learning problem where a forecaster is trying to make a sequence of predictions, such as whether the stock market will go up or down each day.
Every morning, for $T$ days, he solicits the opinions of a number $n$ of experts, who each make up or down predictions. Based on their predictions, the forecaster makes a choice between up and down, then buys or sells accordingly. The goal of the forecaster is to make as few mistakes as possible over time given that the sequence of outcomes and the predictions of the experts may be generated adversarially. 

This is a classic learning problem studied in a large body of literature originating in game theory with the development of fictitious play \cite{Brown49,Robinson51}, Blackwell approachability \cite{blackwell1956}, and Hannan consistency \cite{hannan}, and continued in learning theory under the paradigm of combining expert advice \cite{LW94,Vovk90}.

Perhaps the simplest version of online learning with expert advice is the binary prediction problem, which was illustrated through the stock market scenario. A careful analysis of this problem was done in \cite{BFHHSW97}, which gave bounds on the expected minimax loss of the forecaster and identified $\log_{4}(n)$ as the leading term in the low error regime where the best expert can make at most $o(\log(n))$ mistakes. Many follow-up papers studied binary prediction from different angles, such as \cite{BFHW96}, which designed an algorithm called ``binomial weights'', and \cite{ALW06}, which gave a formulation based on continuous experts that upper bounds the discrete variant.

At the other end of the spectrum there is the hedge setting, a generalization of the binary prediction problem where the forecaster predicts an unknown sequence of elements (points) from an outcome space using the predictions of the experts as inputs, after which nature reveals the next outcome. Then the forecaster incurs a loss that is a function of the point it predicted and the correct point chosen by nature.
One of the best known approaches is the Weighted-Majority algorithm \cite{LW94}, which assigns some initial weights to the experts, follows the advice of each expert with probability given by its weight, and then updates the weights of the experts in every round depending on the quality of the predictions. The average number of mistakes made by the forecaster when using such an algorithm can be bounded by the number of mistakes made by the best expert plus a term of $\sqrt{\log{n}/T}$. More generally, the multiplicative weights update method has been studied in depth and discovered independently in several fields \cite{AHK12,FS97,Young95}. Multiplicative weights gives asymptotically optimal bounds when both $n$ and the horizon go to infinity \cite{BFHHSW97}.

In this paper we study multiclass prediction from expert advice, which can be seen as an intermediate problem between binary prediction and the general hedge setting. An example is predicting (or recognizing) a sequence of digits and has applications such as discrete time signal prediction and choosing portfolios with bounded volatility \cite{Chung94}.
The multiclass prediction problem was studied in various other works such as \cite{FS97}, which bounded its loss using the general hedge framework, and, more generally, for margin classifiers \cite{ASS01} and with bandit feedback \cite{KST08,HK11,CDK09,CG13}.

We briefly survey the previous work on binary prediction as our techniques generalize an approach in \cite{AY} for binary prediction with a perfect expert. 
A basic scenario in binary prediction is when at least one of the experts 
is perfect, that is, predicts correctly every day, but the forecaster does not know which one it is. In this case the problem is to understand how the forecaster should predict in order to minimize the expected number of mistakes (i.e. loss) that he makes in any given number of days. A first known approach is to follow the majority of the leaders (i.e. experts that made no mistake so far), which guarantees the forecaster never makes more than $\log_2 n$ mistakes: each day the forecaster either predicts correctly or eliminates at least
half of experts and, clearly, never eliminates the perfect expert.
This analysis is tight when the minority is right each time and has size nearly equal to that of the majority.

However, an approach known as following a random leader
yields a slightly better guarantee -- see chapter 18 in \cite{AY} for this analysis. For any $n$, the loss of the forecaster is at most $\textrm{H}_n - 1$ in expectation, and 
this analysis is tight when the number of rounds is greater than $n$.
The optimal approach for when the number of experts $n$ is a power of two turns out to be that of using a function of the majority size, where
 the leaders are split on their advice in proportion $(x,1-x)$ with $x \ge 1/2$, follow the majority with probability $p(x)$, where the probability is given by $p(x)=1+\log_4x$.
This algorithm gives an expected loss of at most $\log_{4}{n}$ \cite{AY}
 and opens up the question of obtaining sharp bounds for the loss beyond the binary prediction with a perfect expert scenario. We will be interested in the problem where there is an upper bound $b$ on the number of mistakes made by the best expert in the set and the number of choices is larger than two.

\subsection{Our Results}

We study the multiclass prediction problem from expert advice, where there is an infinite sequence of choices (e.g. digits) drawn from a finite set, revealed one at a time, and a forecaster trying to take the correct decision in each round. After taking its decision, the forecaster learns what the correct decision was and incurs a unit cost for each mistake. The forecaster's goal is to minimize the number of mistakes (loss) made in expectation. 

The forecaster has access to a set of $n$ experts that make predictions on what the correct choice will be and we know that the best expert in this set makes at most $b = b(n)$ mistakes. 
 
The question is understanding how the forecaster should aggregate the opinions of the experts in order to predict and minimize his expected loss (number of mistakes) over time. We make no assumptions on how the sequence is generated (e.g. it may be chosen adversarially adaptive) and 
without loss of generality, each choice in the sequence is chosen from the set $[d] = \{1, \ldots, d\}$, where $d \leq n$.\footnote{This is w.l.o.g. since the number of different opinions is at most $n$.}

\smallskip

Our main result is to give upper and lower bounds for this problem, which are matching within a small additive error. 
In the following, unless specified otherwise, we refer to the forecasting problem where at the beginning there are $n$ experts that made zero errors so far.

\begin{theorem}[Upper Bound]
	Consider a forecaster with $n$ experts, where the best expert makes at most $b= b(n)$ mistakes. Then
	the expected loss of the optimal forecaster is at most 
	$$\log_{4}{(n)} + b \cdot \left[\log_{4}{\left( \log_{2}{(n+1)} \right)} +4 \right] + 1.$$
\end{theorem}

We also give a tighter upper bound when $1 \leq b \leq \ln{(n)}/2$, which is matching the lower bound within an additive term of $O(b)$.

\begin{theorem}[More Precise Upper Bound]
		Consider a forecaster with $n$ experts, where the best expert makes at most $b= b(n)$ mistakes. For any $1 \leq b < \ln{(n)}/2$,
the expected loss of the optimal forecaster is at most 
$$
\left( 1 + 2b/\ln(n)\right) \left( \log_{4}{(n)} +
b \cdot \log_{4}{\left(\log_{2}{n}/b\right)} \right).
$$
\end{theorem}

Our result shows in particular that in the regime where the loss of the best expert is small enough (i.e. $o(\log_4{n})$) the leading term in the expected loss of the forecaster remains $\log_4{n}$.
The upper bound also implies an algorithm for the forecaster. 

The best previously known upper bounds come from several prior works. First, the classic paper on how to use expert advice by \cite{BFHHSW97} also provides an upper bound on the expected loss when the maximum loss of the best expert is known apriori. This bound has the correct leading term of $\log_{4}(n)$ as we also do, but the second order term is not explicit and the analysis is only done for binary prediction.
Second, a bound on our multiclass online learning problem can also be obtained by setting an appropriate learning rate in the multiplicative weights method, which for any $\eta > 0$ has a regret \footnote{The regret is defined as the difference between the expected loss of the forecaster and the loss of the best expert} of at most $\frac{\eta \cdot b + \ln{n}}{1 - e^{-\eta}}$ (see Theorem 2.4, Chapter 2, \cite{BL06-book}). For our setting, optimizing over $\eta$, implies that the upper bound on the expected loss given by multiplicative weights is: $\ln{n} + b \cdot \ln{\left(\frac{e \cdot \ln{n}}{b}\right)} + O\left(\frac{b \cdot (\ln{R})^2}{R}\right)$, for $R > e$ and $1 \leq b \leq \frac{\ln{n}}{R}$.
This bound is a constant factor away from the optimum, since the correct leading term is $\log_{4}(n)$.

\bigskip
In addition to the upper bounds, we also describe a strategy for the adversary that yields the following lower bound.

\begin{theorem}[Lower Bound]
	Consider a forecaster with $n$ experts, where the best expert makes at most $b = b(n)$ mistakes.
	If $b(n) \leq \lfloor \log{n} \rfloor/5$, then any algorithm used by the forecaster has in the worst case an expected loss of at least 
	$$0.5 \cdot \left\lfloor \log_2{n} \right\rfloor  +  0.5 \cdot \left\lfloor \log_2{{\lfloor \log_2{n} \rfloor\choose b}} \right\rfloor - 0.5 \cdot b.$$
\end{theorem}

Note the difference between the upper and lower bound is $O(b)$.
Our upper bound is obtained by establishing first an abstract theorem showing that \emph{any} function that satisfies three conditions is an upper bound on the expected loss of a forecaster that plays optimally against a worst case adversary. Every such function also gives the forecaster a strategy for playing in such a way that its expected loss is upper bounded by $f$. In particular, this can be used to obtain a polynomial time algorithm for estimating the expected loss from any starting configuration. 

Our abstract theorem is given next and is reminiscent of the method for deriving online learning algorithms from a minimax analysis (for previous work using this type of approach, see, e.g., \cite{ALW06,RSS12}). 

We will denote a state by $\vec{k} = (k_0, \ldots, k_b)$, where there are $k_i$ experts that made $i$ mistakes so far. The experts are divided on the prediction of the next choice so that for each option $j$ there is a vector $\vec{k}^j = (k_{0}^j, \ldots, k_b^j)$ to denote that $k_i^j$ experts with $i$ mistakes so far vote for $j$ next. The partition given by the vectors $\vec{k}^1, \ldots, \vec{k}^d$ is said to be a decomposition. 
For each $j$, we denote the successor state obtained if option $j$ turns out to be correct by a $(b+1)$-dimensional vector
$
s_j = s_j(\vec{k}^1, \ldots, \vec{k}^d)$, so that the value at the first coordinate is $k_0^j$, while each coordinate $i > 1$ has value $k_{i}^j + \sum_{v \neq j} k_{i-1}^v$.



\begin{theorem}[Abstract Theorem]
	Let $f : \mathbb{N}^{b+1} \setminus \{(0, \ldots, 0) \} \rightarrow \Re$ be any function that satisfies the properties
	\begin{enumerate}
		\item Boundary: $f(0, \ldots, 0, 1) \geq 0$  
		\item Decomposition property: for every state $\vec{k}$, subset $A \subseteq [d]$ of $a \geq 1$ choices, and decomposition $(\vec{k}^1, \ldots, \vec{k}^d)$ 
		$$
		f(\vec{k}) \geq \frac{a-1}{a} + \frac{1}{a} \cdot \sum_{i \in A} f(\vec{s}_i)
		$$
		where $\vec{s}_i$ is the successor state if choice $i$ is correct.
	\end{enumerate}  
	
	Then the forecaster has a strategy ensuring the expected loss in the worst case is upper bounded by $f$ as follows: from any state $\vec{k}$, observe the decomposition given by the opinions of the experts and the successor states $\vec{s}_1, \ldots, \vec{s}_d$. Then predict choice $i$ with probability 
$p_i = \max\{0, 1 + f(\vec{s}_i) - f(\vec{k})\}$ for all $i < d$ and choice $d$ with probability $p_d = 1 - \sum_{i < d} p_i$.
\end{theorem}

We also obtain a recursive algorithm for computing the minimax loss by iterating over all the possible subsets $A$ of choices and decompositions.

\begin{theorem}[Exact Algorithm]
	The function $\ell$ that represents the minimax loss of the forecaster is given by the following recursion:
	$$
	\ell(\vec{k}) = 
	\max_{A, \vec{k}^1, \ldots, \vec{k}^d} \frac{a-1}{a} + \frac{1}{a} \cdot \sum_{i \in A} \ell(\vec{s}_i)
	$$
	where $\vec{s}_i$ is the successor state of decomposition $(\vec{k}^1, \ldots, \vec{k}^d)$ if choice $i$ is correct and
	the base case is $\ell(0, \ldots, 0, 1) = 0$.
	This gives an algorithm for computing the exact value of $\ell$.
\end{theorem}

An example of the loss function for the two dimensional case with $b=1$ mistakes is given in the next figure.

\begin{figure}[h!]
\centering
\includegraphics[scale = 0.5]{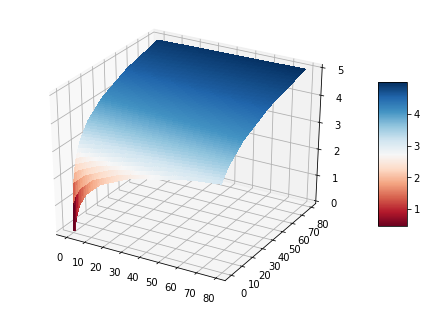}
\label{fig:loss_intro}
\caption{Minimax loss function for $b=1$ mistakes starting from state $(i,j)$, for $i,j=0 \ldots 80$.}
\end{figure}

Finally, we compute the exact value for the expected loss in the case of a perfect expert. This value is exactly $\log_4{n}$ when $n$ is a power of $2$ and is obtained by interpolating between the values of powers of two otherwise.

\begin{theorem}[Perfect Expert]
	The expected loss of a forecaster with $n$ experts, one of which is perfect, is 
	$$\ell(n) = 0.5 \cdot \left(k-1+\frac{n}{2^k}\right)$$
	where $k$ is such that $2^k \le n < 2^{k+1}$.
\end{theorem}

\subsection{Further Related Work}

The online learning problem can alternatively be seen as a zero-sum game between the learner (forecaster) and Nature (the adversary). The connection between online learning and games has been explored in a variety of settings, such as calibrated learning and convergence to correlated equilibria \cite{FV97}, online convex optimization games \cite{ABRT08,AABR09} , where minimax duality is exploited to obtain bounds on regret, online linear optimization games \cite{MO14} where the minimax optimal regret is characterized together with the minimax optimal algorithm and adversary, and drifting games \cite{MS10}, where a group of experts make continuous predictions and there is a known upper bound on the number of mistakes that the best expert can make. 

The expert prediction framework has been used to study more generally problems in online convex optimization \cite{HKKA06,Zinkevich03}, where a decision-maker makes
a sequence of decisions (points in Euclidean space), so that after each decision a cost function 
is revealed.
The connection between Blackwell approachability and online learning was also analyzed more recently, such as in \cite{ABH11} which showed an equivalence between the two notions. Many other online learning problems for various notions of regret have been studied (see, e.g. \cite{BK99,ABR07,RST11,GPS14,RST10,ADT12,HKW95,AWY08,Koolen13}, the survey on online algorithms by \cite{Blum98}, and the books on machine learning \cite{SSBD14,Mohri,BL06-book} and (algorithmic) game theory \cite{AY,NRTV}).
%
%

\section{Preliminaries}

We are given a set of $n \geq 1$ experts and an upper bound $b = b(n)$ on the number of mistakes that the best expert can make. The exact value of the number of mistakes made by the best expert will only be used in the lower bound analysis. 
We will refer to nature as the adversary, which will set the predictions of the experts and correct choice in each round. Nature's strategy may be randomized.

\smallskip

We will define the current state to keep track of the number of experts with at most $b$ mistakes so far.


\medskip

\begin{definition}[State]
For any number of mistakes $b$, the current \emph{state} is given by a $(b+1)$-dimensional vector $\vec{k} = (k_0, \ldots, k_b)$, so that $k_i$ represents the number of experts that made $i$ mistakes so far.
\end{definition}

Given any history $\mathcal{H}$ of play, consisting of all the rounds, with the strategies and outcomes for both forecaster and nature (adversary) for each round, we denote by $\ell(\mathcal{H})$ the expected loss\footnote{The expected loss is the expected number of mistakes.} made by an optimal forecaster playing against an optimal adversary when the history so far is given by $\mathcal{H}$.


\begin{definition}[Decomposition]
Given a state $\vec{k} = (k_0, \ldots, k_b)$, a \emph{decomposition} $(\vec{k}^1, \ldots, \vec{k}^d)$ of $\vec{k}$ is a partition of the experts into vectors $\vec{k}^j = (k_0^j, \ldots, k_b^j)$, where $k_i^j$ experts with $i$ mistakes so far vote for choice $j$ next, for each $i \in \{ 0, \ldots, b \}$, and $\sum_{j=1}^d k_i^j = k_i$.
\end{definition}

Note that if the adversary splits the experts using a decomposition $(\vec{k}^1, \ldots, \vec{k}^d)$ and decides in some arbitrary way the next choice, then there are only $d$ possible states that can be reached from $\vec{k}$ in one round of prediction, depending on which choice turns out to be the correct one.


\begin{definition}[Successor State]
A state $\vec{m} = (m_0, \ldots, m_b)$ is a \emph{successor} of a state $\vec{k} = (k_0, \ldots, k_b)$ if there exists a decomposition $(\vec{k}^1, \ldots, \vec{k}^d)$ of $\vec{k}$ with the property that $\vec{m} = \vec{s}_j(\vec{k}^1, \ldots, \vec{k}^d)$ for some $j \in [d]$, where 
\begin{itemize}
\item $\vec{s}_j = \vec{s}_j(\vec{k}^1, \ldots, \vec{k}^d)$ is the next state from $\vec{k}$ if choice $j$ is correct, 
\item the first coordinate of $\vec{s}_j$ is $k_0^j$, and 
\item the value at each other coordinate $i > 1$ of $\vec{s}_j$ is $k_{i}^j + \sum_{v \neq j} k_{i-1}^v$.
\end{itemize}
\end{definition}
In other words, a state $\vec{m}$ is a successor of $\vec{k}$ if $\vec{m}$ is reachable from $\vec{k}$ in one round of prediction. More generally, we can define reachability using sequences of successors.

\begin{definition}[Reachable State]
	
A state $\vec{r}$ is \emph{reachable} from state $\vec{k}$ if there exists a sequence of states $\vec{s}_1, \ldots, \vec{s}_t$ so that $s_{i+1}$ is a successor of $s_i$ for each $i = 1 \ldots t-1$ and $\vec{s}_t = \vec{r}$.
\end{definition}

The definitions are illustrated through the next example. 

\begin{example}
Let the state be $\vec{k} = (3,1)$ with $d=2$ options. Suppose the decomposition of the experts is $\vec{k}^1 = (1,0)$---i.e. option 1 is voted by one expert with zero mistakes so far and zero experts with one mistake so far---and $\vec{k}^2 = (2,1)$. Then if the correct choice turns out to be 
\begin{itemize}
\item $1$: the successor state is $\vec{s}_1(\vec{k}^1, \vec{k}^2) = (k_0^1, k_1^1 + k_0^2) = (1,2)$. 
\item $2$: the successor state is $\vec{s}_2(\vec{k}^1, \vec{k}^2) = (k_0^2, k_1^2 + k_0^1) = (2,1)$.
\end{itemize}
\end{example}

\section{Properties of the Optimal Loss}

In this section we identify several basic properties of the minimax loss function.

Recall that a state $\vec{k}$ was defined to encode only the experts with at most $b$ mistakes so far. The next lemma shows that the loss function is independent of any other details of the history beyond the current state since both the forecaster and the adversary have a strategy ensuring the same expected loss.

\begin{lemma} \label{lem:general}
Let $\mathcal{H}$ be any history. Suppose the current state is $\vec{k} = (k_0, \ldots, k_b)$, where $\vec{k} \neq (0, \ldots, 0, 1)$. Then
\begin{equation} \label{eq:general}
\ell(\mathcal{H}) = \max_{\vec{p},(\vec{k}^1,\ldots, \vec{k}^d)}
	\left(1 - \max_{i \in [d]} p_i  +  \sum_{j=1}^d p_j \cdot \ell\left(\vec{s}_j\left(\vec{k}^1, \ldots, \vec{k}^d\right)\right)\right)
\end{equation}
	where $(\vec{k}^1, \ldots, \vec{k}^d)$ is a decomposition of the experts and $\vec{p} = (p_1, \ldots, p_d)$ a probability vector.
\end{lemma}
\begin{proof}
Consider the adversary's strategy of selecting a probability vector $\vec{p} = (p_1, \ldots, p_d)$, a decomposition $(\vec{k}^1, \ldots, \vec{k}^d)$ of the experts with at most $b$ mistakes, and choosing each option $i \in [d]$ with probability $p_i$.

The expected loss incurred by the forecaster when deterministically choosing a fixed option $i$ is $1 - p_i + \sum_{j=1}^d p_j \cdot \ell\left(\vec{s}_j\left(\vec{k}^1, \ldots, \vec{k}^d\right)\right)$. 
Thus by predicting the choice $i$ with the highest probability of appearing next, the forecaster can ensure its expected loss is at most
$$1 - \max_{i \in [d]} p_i + \sum_{j=1}^d p_j \cdot \ell\left(\vec{s}_j\left(\vec{k}^1, \ldots, \vec{k}^d\right)\right)$$
Note this is true for any history compatible with the current state.

Since the adversary may have better strategies available, we obtain that $\ell(\mathcal{H}) \geq 
1 - \max_{i \in [d]} p_i + \sum_{j=1}^d p_j \cdot \ell\left(\vec{s}_j\left(\vec{k}^1, \ldots, \vec{k}^d\right)\right)$.
The adversary's problem is then to select the worst case probability vector and decomposition, given that the forecaster will only predict from choices that minimize its loss. By induction on the state, the required identity gives the value of the zero-sum game between the forecaster and adversary.
\end{proof}

From now on, we will override notation and write $\ell(\vec{k})$ to denote the minimax loss for any history $\mathcal{H}$ compatible with the state $\vec{k}$.

\medskip

For the one dimensional problem corresponding to the case of a perfect expert we identify the exact expression for the loss function. This expression will be used as a base case for establishing the structure of the optimal decomposition for any number $d$ of options.

\begin{theorem}[Perfect Expert] \label{lem:exact_perfect}
	The expected loss made by a forecaster with $n$ experts, one of which is perfect, is $$\ell(n) = 0.5 \cdot \left(k-1+\frac{n}{2^k}\right)$$ where $k$ is such that $2^k \le n < 2^{k+1}$.
\end{theorem}

In order to prove the theorem we will develop several lemmas. 

\begin{lemma} \label{lem:f_1d}
Let $f : \mathbb{N} \rightarrow \Re$ be given by $f(n) = 0.5 \cdot \left(k-1+\frac{n}{2^k}\right)$, where $k$ is such that $2^k \le n < 2^{k+1}$. Then the function $f$ satisfies the properties:
\begin{itemize}
	\item $f(n) = 0.5 + 0.5 f\left(\left\lfloor n/2 \right\rfloor\right) + 0.5 f\left(\left\lceil n/2 \right\rceil\right)$ for all $ n \in \mathbb{N}$.
	\item For each $n, a, n_1, \ldots, n_a \in \mathbb{N}$ such that $n = \sum_{i=1}^a n_i$ with $n_i \geq 1$ for all $i$,
	$$
	f(n) \geq 1 - \max_{i=1}^a p_i + \sum_{i=1}^a p_i \cdot f(n_i)
	$$
	for every probability vector $(p_1, \ldots, p_a)$.
\end{itemize}
\end{lemma}
\begin{proof}
We show the inequality by induction on $a$ and $n$. 
For $a = 2$, by separating the case of $n$ even and odd, we obtain that the function $f$ satisfies the
identity $$f(n) = 0.5 + 0.5 f\left(\left\lfloor n/2 \right\rfloor\right) + 0.5 f\left(\left\lceil n/2 \right\rceil\right).$$ Moreover, it can be verified that this function satisfies the inequality $f(n) \geq 1 - \max\{p_1, p_2\} + p_1 \cdot f(n_1) + p_2 \cdot f(n_2)$ for all $n_1, n_2 \geq 1$ with $n = n_1 + n_2$ and probability vector $(p_1, p_2)$.

Assume it holds for $a-1$. If $n=a$, then $n_i = 1$ for all $i$, so $f(n_i) = 0$ for all $i$ and the inequality follows by the definition of $f$ and the fact that $1 - \max_{i=1}^{a} p_i \leq 1 - 1/a$. 
If there exists $i$ such that $n_i = 1$, then this reduces the value of $a$ and follows from the induction hypothesis.
Otherwise, $n_i > 1$ for all $i$.

For each $i = 1 \ldots a$, write $n_i = k_i + l_i$, where 
$$
f(n_i) = \frac{1 + f(k_i) + f(l_i)}{2}
$$
From the induction hypothesis, we have that 
\begin{align}
f\left(\sum_{i=1}^a k_i\right) \geq 1 - \max_{i=1}^a p_i + \sum_{i=1}^a p_i \cdot f(k_i)\\
f\left(\sum_{i=1}^a l_i\right) \geq 1 - \max_{i=1}^a p_i + \sum_{i=1}^a p_i \cdot f(l_i)
\end{align}
Then we obtain the following inequalities
\begin{eqnarray*}
f(n) & \geq & \frac{1 + f\left(\sum_{i=1}^a k_i\right) + f\left(\sum_{i=1}^a l_i\right)}{2} \\
& \geq & \frac{1}{2} + (1 - \max_{i=1}^a p_i) + \sum_{i=1}^a \left[\frac{f(k_i) + f(l_i)}{2} \right] \\
& = & 1 - \max_{i=1}^a p_i + \sum_{i=1}^a p_i \cdot \left[ \frac{1 + f(k_i) + f(l_i)}{2} \right] \\
& =& 1 - \max_{i=1}^a p_i + \sum_{i=1}^a p_i \cdot f(n_i)
\end{eqnarray*}
as required.
\end{proof}

\begin{proof}[Proof of Theorem \ref{lem:exact_perfect}]
The proof follows immediately by observing that the function $f$ defined in Lemma \ref{lem:f_1d} satisfies the same constraints required from the optimal function (Lemma \ref{lem:general}) and one of the constraints is tight, when dividing $n$ into two sets as equal as possible---of sizes $\lfloor n/2 \rfloor$ and $\lceil n/2 \rceil$---and setting the probabilities for the two sets equal to $1/2$.
\end{proof}

	
Going beyond the perfect expert scenario, we will show that the adversary always has an optimal strategy of the following form: choose a set $A$ of options, divide the experts into groups so that each group $i$ votes for option $i$ in the set $A$, and then select the next choice from $A$ giving them equal probability.

\begin{definition}[Domination]
Let $\vec{k}$ be any state. Suppose the adversary promises to select the next choice from a non-empty subset $A \subseteq [d]$.
Then a decomposition $\vec{p} = (\vec{p}^1, \ldots, \vec{p}^d)$ \emph{(weakly) dominates} another decomposition $\vec{q} = (\vec{q}^1, \ldots, \vec{q}^d)$ with respect to this adversary if for each $i \in A$ the successor $\vec{s}_i(\vec{q})$ is reachable from $\vec{s}_i(\vec{p})$ or equal to it.
\end{definition}
Note that the promise of the adversary to only select a choice in $A$ only holds for the next choice.
The next lemma will be useful for establishing that when the adversary promises to select the next choice from a subset $A$, it is sufficient to study decompositions in which no expert predicts a choice outside $A$.

\begin{lemma} \label{lem:major}
Let $\vec{k}$ be any state. Suppose the adversary promises to select the next choice from a non-empty subset $A \subseteq [d]$. Then any decomposition $\vec{q} = (\vec{q}^1, \ldots, \vec{q}^d)$ is weakly dominated by a decomposition $\vec{p} = (\vec{p}^1, \ldots, \vec{p}^d)$ in which no expert predicts choices outside $A$ (i.e. with $\vec{p}^i = \vec{0}$ for each $i \not \in A$).
\end{lemma}
\begin{proof}
W.l.o.g., $A = \{1, \ldots, a\}$ for some $a \in [d]$. Given an arbitrary decomposition $\vec{q} = (\vec{q}^1, \ldots, \vec{q}^d)$, consider a decomposition $\vec{p} = (\vec{p}^1, \ldots, \vec{p}^d)$ with the properties:
\begin{itemize}
\item $\vec{p}^v = \vec{0}$ for each $v \not \in A$.
\item $p_i^v = q_i^v$ for each $v \in A \setminus \{1\}$ and $i \in \{ 0, \ldots, b \}$.
\item $p_i^1 = q_i^1 + \Delta_i$, where $\Delta_i = \sum_{w \not \in A} p_i^w$, for $i \in \{ 0, \ldots, b\}$.
\end{itemize}

For each option $v \in A \setminus \{1\}$, the successor states are clearly the same under the two decompositions, i.e. $\vec{s}_v(\vec{p}) = \vec{s}_v(\vec{q})$. For the first choice, the successor state under $\vec{p}$ is $\vec{r} = \vec{s}_1(\vec{p})$, where $r_0 = q_0^1 + \Delta_0$ and $r_i = p_{i}^1 + \sum_{v \neq 1} p_{i-1}^v = (q_i^1 + \Delta_i) + (k_{i-1} - q_{i-1}^1 - \Delta_{i-1})$
for $i \in \{1, \ldots, b\}$. On the other hand, the successor state corresponding to the first choice under decomposition $\vec{q}$ is $\vec{t} = \vec{s}_1(\vec{q})$, where $t_0 = q_0^1$ and $t_i = q_i^1 + \sum_{v \neq 1} q_{i-1}^v = q_i^1 + (k_{i-1} - q_{i-1}^1)$ for $i \in \{1, \ldots, b\}$.
Then state $\vec{t}$ can be obtained from $\vec{r}$ in one round by having $\Delta_i$ experts with $i$ mistakes so far predict choice $2$ for each $i$ while everyone else predicts choice $1$, and then selecting choice $1$.
\end{proof}

Next we show that the adversary has an optimal strategy that involves restricting the set of choices to some subset $A$ and choosing the next option uniformly at random from $A$. We simultaneously show that the loss function is strictly monotonic with respect to successors.

\begin{lemma} \label{lem:minineq}
Let $\vec{k}$ be any state other than $(0, \ldots, 0, 1)$. Then
\begin{description}
\item[(i)] there exists a subset $A$ of $a \geq 2$ choices and a decomposition $(\vec{k}^1, \ldots, \vec{k}^d)$ of the experts where $\vec{k}^i = \vec{0}$ for each $i \not \in A$, so that 
$$\ell(\vec{k}) = \frac{a-1}{a}  + 
\frac{1}{a} \cdot \left( \sum_{i \in A} \ell(\vec{s}_i)\right)$$
where $\vec{s}_i$ is the successor state if choice $i$ is correct.
\item[(ii)] $\ell(\vec{k}) > \ell(\vec{r})$ for every state $\vec{r} \neq \vec{k}$ that is reachable from $\vec{k}$.
\end{description}
\end{lemma}
\begin{proof}

We proceed by induction on the state $\vec{k}$. For the base case we must verify two subcases:
\begin{itemize}
	\item $\vec{k} = (0, \ldots, 0, k_b)$. In this case there are exactly $k_b$ experts left and each of them has made $b$ mistakes. From such a state the problem is equivalent to the perfect expert prediction problem. Both conditions $(i)$ and $(ii)$ follow from the properties of the loss function established in Lemma \ref{lem:f_1d} and Theorem \ref{lem:exact_perfect}.

	\item $\vec{k} = (0, \ldots, 0, 1, 0)$. In this case there is only one expert left that has made $b-1$ mistakes so far, which is equivalent to having one expert that makes one mistake. Then no matter what decomposition the adversary chooses, one of the successor states will be the same as the current state, $\vec{k}$, while all the other successor states will be states in which there is only one expert that is moreover perfect, so no further mistakes can be made from these states. The general strategy of the adversary is to pick a set $A$ of size $a$ and a probability vector $(p_1, \ldots, p_a)$, have the one expert predict a fixed option (say $i^*$), and then pick the next choice as $i \in A$ with probability $p_i$. The expected loss is then $\ell(\vec{k}) = 1 - \max_{i=1}^a p_i + p_{i^*} \cdot \ell(\vec{k})$, which implies that $\ell(\vec{k}) \leq 1$. By having the expert predict choice $1$, and flipping a fair coin to decide whether the correct choice is $1$ or $2$, the adversary can ensure a loss of $\ell(\vec{k}) \geq 1/2 + 1/2 \cdot \ell(\vec{k})$, i.e. $\ell(\vec{k}) = 1$.
\end{itemize}

Assume that properties \textbf{(i)} and \textbf{(ii)} hold for all the states reachable from $\vec{k}$ in at least one round, and show they also hold for $\vec{k}$. We start with \textbf{(ii)} and observe that it is sufficient to verify $\ell(\vec{k}) > \ell(\vec{r})$ for every state $\vec{r} = (r_0, \ldots, r_b)$ that is reachable from $\vec{k}$ by having only one expert make a mistake in state $\vec{k}$ (in this case $\vec{r}$ is a successor of $\vec{k}$). Let $i$ be the number of mistakes this expert has accumulated when reaching $\vec{k}$. Let 
$\vec{h} = (\vec{h}^1, \ldots, \vec{h}^d)$ be an adversary optimal decomposition of the experts in state $\vec{r}$, so that $h_{j}^c$ experts with $j$ mistakes so far recommend choice $c$ next, and there is a set $A$ of $2 \leq a \leq d$ options so that the next choice is selected from $A$ with equal probability:
$$
\ell({\vec{r}}) = \frac{a-1}{a} + \frac{1}{a} \cdot \left( \sum_{i \in A} \ell(\vec{s}_i(\vec{h}^1, \ldots, \vec{h}^d)) \right)
$$
where $\vec{s}_i(\vec{h}^1, \ldots, \vec{h}^d)$ is the successor state from decomposition $(\vec{h}^1, \ldots, \vec{h}^d)$ if choice $i$ is correct.
Without loss of generality, suppose $1 \in A$.

Given that state $\vec{r}$ is obtained from $\vec{k}$ by only having one expert with $i$ mistakes so far predict incorrectly at $\vec{k}$, we have the following properties for the decomposition $\vec{h}$:
\begin{enumerate}
\item $\sum_{c=1}^d h_j^c = k_j$ for each $j \in \{1, \ldots, b\} \setminus \{i,i+1\}$
\item $\sum_{c = 1}^d h_i^c = k_i - 1$
\item if $i < b$, then $\sum_{c=1}^d h_{i+1}^c = k_{i+1} + 1$. Since $\sum_{c=1}^d h_{i+1}^c \geq 1$, there must exist $c$ such that $h_{i+1}^c$ is strictly positive, so w.l.o.g. $h_{i+1}^1 > 0$. (Note the case $i=b$ is already handled through the first property.)
\end{enumerate}

Let $\vec{z} \in \Re^{b+1}$ be defined as:
\begin{enumerate}
	\item $z_j = 0$ for each $j \in \{1, \ldots, b\} \setminus \{i,i+1\}$
	\item $z_i = 1$ 
	\item if $i < b$, then $z_{i+1} = -1$.
\end{enumerate}
Then $\vec{h}' = (\vec{h}^1 + \vec{z},\vec{h}^2, \ldots, \vec{h}^d)$ is a valid decomposition of the state $\vec{k}$ with the property that the experts predict only choices in $A$. Then we obtain 

\begin{eqnarray*}
	\ell(\vec{k}) & \geq& 
	\frac{a-1}{a} + \frac{1}{a} \cdot \left( \sum_{c \in A} \ell(\vec{s}_c(\vec{h}')) \right)\\
	& > & \frac{a-1}{a} + \frac{1}{a} \cdot \left( \sum_{c \in A} \ell(\vec{s}_c(\vec{h})) \right)\\
	& = & \ell(\vec{r}),
\end{eqnarray*}
where $\vec{s}_c(\vec{h}')$ and $\vec{s}_c(\vec{h})$ are the successor states from decomposition $\vec{h}'$ and $\vec{h}$, respectively, if choice $c$ turns out correct. The first inequality follows from Lemma \ref{lem:general}, while the second inequality and the identity hold by the induction hypothesis for 
$\vec{r}$. Thus $\ell(\vec{k}) > \ell(\vec{r})$ as required.
\medskip

We now check that \textbf{(i)} also holds for $\vec{k}$. Let $\vec{\alpha} = (\alpha_1, \ldots, \alpha_d)$ be the probability vector used by the adversary, so that the next choice is set to $i$ with probability $\alpha_i$, and $(\vec{k}^1, \ldots, \vec{k}^d)$ a decomposition of the experts. 
Then if the forecaster predicts choice $i$ next, the expected loss from now on including this round is
$$v_i(\alpha)= (1 - \alpha_i) + \sum_{i=1}^d \alpha_i \cdot \ell(\vec{s}_i(\vec{k}^1, \ldots, \vec{k}^d))$$ 
%
%
The forecaster will select the option that minimizes the expected loss, which gives a cost of
$\min\{v_1(\vec{\alpha})$, $\ldots$, $v_d(\vec{\alpha})\}$ for a fixed $\vec{\alpha}$. The adversary's problem is then to find $\vec{\alpha}$ that maximizes this value. 
If $\alpha_i \geq \alpha_j$ for all $j$, this is given by the following linear program, where
$\ell_j = \ell(\vec{s}_j(\vec{k}^1, \ldots, \vec{k}^d))$ is the loss from the successor state obtained when choice $j$ is correct given the decomposition $(\vec{k}^1, \ldots, \vec{k}^d)$:
\begin{align} \label{minilp}
	\max & \; \; \; 1 - \alpha_i + \left( \sum_{j=1}^d \alpha_j \cdot \ell_j \right) \notag\\
	\mathrm{s. t.} & \; \; \; \alpha_i \geq \alpha_j \; \; \forall j \in [d] \notag \\
	& \; \; \;\alpha_j \geq 0 \; \; \forall j \in [d] \notag\\
	& \; \; \;\sum_{j=1}^d \alpha_j = 1 
\end{align}

First note that due to the strict monotonicity property of the loss function, 
it cannot be the case that setting $\alpha_j=1$ for some $j$ is the optimal solution, unless the successor state for option $j$ is the same as the current state (i.e. $s_j(\vec{k}^1, \ldots, \vec{k}^d) = \vec{k}$). However, in the latter case, the adversary leaves the current state 
as it is with probability $1$, and such rounds can be removed. Thus the decomposition and probability vector used in the optimal solution will have the property that at least two successor states are reached with non-zero probability.

The linear program in (\ref{minilp}) attains an optimal solution at an extreme point. Consider any optimal solution that does not have the required form, i.e. for which there does not exist $a \in \{2, \ldots, d\}$ such that all the non-zero entries of $\vec{\alpha}$ are equal to $1/a$. Suppose the sorted entries in the probability vector $\vec{\alpha}$ are $\alpha_{j_1} = \ldots = \alpha_{j_v} > \alpha_{j_{v+1}} \geq \ldots \geq \alpha_{j_{n}} \geq 0$, where  $\alpha_{j_{v+1}} > 0$. Let $\epsilon> 0$ be small enough so that $\epsilon < \min\{ \frac{v}{v+1} \cdot (\alpha_{j_v} - \alpha_{j_{v+1}}), \alpha_{j_{v+1}}\}$. 
Consider alternative probability vectors $\vec{\beta}$ and $\vec{\gamma}$ given by 
\begin{itemize}
	\item $\beta_{j_w} = \gamma_{j_w} = \alpha_{j_w}$ for all $v+2 \leq w \leq n$.
	\item $\beta_{j_{v+1}} = \alpha_{j_{v+1}} - \epsilon$ and $\gamma_{j_{v+1}} = \alpha_{j_{v+1}} + \epsilon$.
	\item $\beta_{j_{w}} = \alpha_{j_w} + \epsilon / v$ and $\gamma_{j_w} = \alpha_{j_w} - \epsilon / v$ for all $ 1 \leq w \leq v$.
\end{itemize}
Note that by construction $\vec{\beta}$ and $\vec{\gamma}$ satisfy the constraints of the linear program and moreover $\vec{\alpha} = (\vec{\beta} + \vec{\gamma})/2$. Then $\vec{\alpha}$ is not an extreme point. Thus any extreme point that is optimal has the property that all the non-zero entries in the probability vector are equal to $1/a$ for some $a \in \{2, \ldots, d\}$. 

It follows that the adversary has an optimal strategy of picking a subset $A$ of at least two choices and assign equal probability to all the elements in $A$. Moreover, by Lemma \ref{lem:major}, we can assume that $\vec{k}^i = \vec{0}$ for all $i \not \in A$ since any decomposition that does not have this property is dominated by a decomposition that does have the property and achieves at least as high of an objective value for the same choice of probabilities. This establishes property \textbf{(i)} for $\vec{k}$.
\end{proof}

Lemma \ref{lem:minineq} implies an algorithm for computing the exact value of the loss $\ell$.

\begin{theorem}[Exact Algorithm]
	The function $\ell$ that represents the minimax loss of the forecaster is given by the following recursion:
$$
\ell(\vec{k}) = 
\max_{A,(\vec{k}^1, \ldots, \vec{k}^d)} \frac{a-1}{a} + \frac{1}{a} \cdot \left(\sum_{i \in A} \ell(\vec{s}_i)\right)
$$
	where $A$ is a subset of $a \geq 2$ choices, $(\vec{k}^1, \ldots, \vec{k}^d)$ is a decomposition of the experts where $\vec{k}^i = \vec{0}$ for each $i \not \in A$, and 
$\vec{s}_i$ is the successor state given this decomposition when choice $i$ is correct.
The base case is $\ell(0, \ldots, 0, 1) = 0$.
\end{theorem}
\begin{proof}
The proof follows by using property \textbf{(i)} of the function $\ell$ from Lemma \ref{lem:minineq}.
\end{proof}

\section{Lower Bound}

In this section we prove a lower bound on the expected loss and start by defining the following functions.

\begin{definition} \label{def:gi}
	Let $(n, 0, \ldots, 0)$ be an initial state. Define $g_i(t, n)$ as the random variable for the number of experts that made exactly $i$ mistakes at time $t$, given that the sequence of bits is produced by an adversary that works as follows:
	\begin{itemize}
		\item in every state $\vec{k} = (k_0, \ldots, k_b)$, compute a decomposition $(\vec{k}^0, \vec{k}^1)$ with the property that for all $j = 0 \ldots b$, if $k_j$ is 
	\begin{itemize}
		\item even, then set $k_j^0 = k_j^1$
		\item odd, then split the experts with $j$ mistakes as evenly as possible while alternating when $k_j^0$ or $k_j^1$ is larger. Formally, let $r$ be the number of indices $\ell < j$ for which $k_{\ell} \in 2 \mathbb{Z}$. Set $k_j^0 = k_j^1 - 1$ if $r$ is even and $k_j^0 = k_j^1 + 1$ if $r$ is odd.
	\end{itemize} 
\item flip a fair coin to decide whether the next bit is $0$ or $1$.
	\end{itemize}
\end{definition}

For each coordinate $i = 0 \ldots b$, we will be interested in finding the smallest time $t_i^* \geq 0$ with the property that $g_i(t, n) \leq 1$ for all $t \geq t_i^*$.

\begin{lemma} \label{lem:expected}
For each initial number of experts $n \geq 1$ the following inequality holds:
$$
\ell(n, 0, \ldots, 0) \geq t_b^*/2,$$
where $t_b^*$ is the smallest time with the property that $g_b(t, n) \leq 1$ for all $t \geq t_b^*$.
\end{lemma}
\begin{proof}
First note that in any state $\vec{k} = (k_0, \ldots, k_b)$, if the decomposition chosen by the adversary is $(\vec{k}^0, \vec{k}^1)$, if the forecaster predicts $0$ with probability $p$ and $1$ with probability $1-p$, then the expected loss starting from state $\vec{k}$, including the current round, is 
\begin{eqnarray*}
\ell(k_0, \ldots, k_b)	& = & 0.5 \left(1 - p + \ell(s_0(\vec{k}^0, \vec{k}^1))\right) + 0.5 \left(p + \ell(s_1(\vec{k}^0, \vec{k}^1))\right)\\
& = & 0.5 \left(1 + \ell(s_0(\vec{k}^0, \vec{k}^1)) + \ell(s_1(\vec{k}^0, \vec{k}^1)) \right)
\end{eqnarray*}
where we have used the fact that the adversary selects $\vec{k}^0$ or $\vec{k}^1$ as the experts with the correct opinion with probability $1/2$ (Lemma \ref{lem:minineq}). Thus the expected loss of the forecaster per round is $1/2$. If $t_b^*$ is chosen as in the statement of the lemma, then the number of rounds in which the forecaster makes a mistake is greater than or equal to $t_b^*$, which implies
$
\ell(n, 0, \ldots, 0) \geq t_b^* / 2
$ as required.
\end{proof}
 A lower bound on the expected loss when starting from the state $(n, 0, \ldots, 0)$ will be given by $t_b^*/2$, so it will be sufficient to estimate $t_b^*$.

\begin{theorem}\label{thm:lb_sharp}
	Consider a forecaster with $n$ experts, where the best expert makes at most $b = b(n)$ mistakes.
	If $b(n) \leq \lfloor \log{n} \rfloor/5$, then any algorithm used by the forecaster has in the worst case an expected loss of at least 
	$$0.5 \cdot \left\lfloor \log_2{n} \right\rfloor  +  0.5 \cdot \left\lfloor \log_2{{\lfloor \log_2{n} \rfloor\choose b}} \right\rfloor - 0.5 \cdot b.$$
\end{theorem}
\begin{proof}
Let $k = \lfloor \log_2{n} \rfloor$. Since increasing the number of initial experts can only (weakly) increase the value of the $g_i$ functions at any given point in time, we have that $g_i(t,n) \geq g_i(t,2^k)$ for all $i = 0 \ldots b$. Moreover, since the adversary's strategy from any state is to divide each set of experts with $i$ mistakes as evenly as possible, and finding decompositions of every state $\ell$ that are as even as possible, alternating whether $\ell_i^0$ or $\ell_i^1$ is larger (in case of odd set sizes), the following inequality holds for any $i = 1 \ldots b$:
\begin{align} \label{eq:gi_ineq}
g_i(t,2^k) & \geq \min\left\{\left\lfloor \frac{g_{i-1}(t-1,2^k)}{2} \right\rfloor + \left\lceil \frac{g_{i}(t-1,2^k)}{2} \right\rceil, 
\left \lceil \frac{g_{i-1}(t-1,2^k)}{2} \right \rceil + 
\left \lfloor \frac{g_{i}(t-1,2^k)}{2} \right \rfloor 
\right\} \notag \\
& \geq \frac{g_{i-1}(t-1,2^k)}{2} + \frac{g_{i}(t-1,2^k)}{2} - 0.5 
\end{align}

For $i=0$, since the set of experts that made $0$ mistakes so far is exactly halved each time when starting from $2^k$, we get that $g_0(t, n) \geq g_0(t, 2^k) \geq 2^{k-t}$. Thus $t_0^* \geq k$.

For each $i$, let $h_i(t,m)= 2^t \cdot g_i(t,m)$. Multiplying both sides of inequality (\ref{eq:gi_ineq}) by $2^t$, we get 
\begin{equation} \label{eq:hi_simple}
h_i(t, 2^k) \geq h_{i-1}(t-1,2^k) + h_{i}(t-1,2^k) - 2^{t-1}
\end{equation}

We show by induction on $i$ the following inequality:
\begin{equation} \label{eq:hi}
h_i(t,2^k) \geq 2^{k} \cdot {t \choose i} - 2^{t+i} + 2^{t}
\end{equation}
For $i=0$, inequality (\ref{eq:hi}) is equivalent to $h_0(t, 2^k) \geq 2^{k} -2^{t} + 2^{t} = 2^k$, which holds given the previous observation that $g_0(t,2^k) \geq 2^{k-t}$. Assume condition (\ref{eq:hi}) holds for $i-1$. Using inequality (\ref{eq:hi_simple}), we obtain
\begin{align} \label{eq:recurse}
h_i(t,2^k) & \geq h_{i-1}(t-1,2^k) + h_i(t-1,2^k) - 2^{t-1} \notag\\
& \geq \left( 2^{k} \cdot {t-1 \choose i-1} - 2^{(t-1)+(i-1)} + 2^{t-1}
\right) + h_i(t-1,2^k) - 2^{t-1} \notag \\
& = 2^{k} \cdot {t-1 \choose i-1} - 2^{t+i-2} + h_i(t-1,2^k) 
\end{align}
Note $h_i(i-1,2^k) = 0$ since at time $i-1$ there are no experts that made $i$ mistakes when starting from $n$ experts with zero mistakes. Also note condition (\ref{eq:hi}) holds for $t < i$, so we can assume $t \geq i$.
Summing inequality (\ref{eq:recurse}) for $j = i \ldots t$ and cancelling out the terms, we get
\begin{align}
h_i(t, 2^k) & \geq 2^k \cdot \left[\sum_{j=i-1}^{t-1} {j \choose i-1} \right]
+ h_i(i-1,2^k) - \left(2^{t+i-2} + \ldots + 2^{i+i-2} \right) \notag \\
& = 2^k \cdot {t \choose i} - 2^{t+i-1} + 2^{2i-2} \notag\\
& \geq 2^{k} \cdot {t \choose i} - 2^{t+i} + 2^{t},\notag
\end{align}
where the last inequality holds if and only if $2^{t+i-1} + 2^{2i-2} \geq 2^{t}$, which is true for all $i \geq 1$. Thus condition (\ref{eq:hi}) holds for $i$. We obtain that $
g_i(t,2^k) \geq 2^{k-t} \cdot {t \choose i} - 2^{i} + 1
$ for all $i = 0 \ldots b$.
Taking $i = b$, it follows that
$$
g_b(t,2^k) > 2^{k-t} \cdot {t \choose b} - 2^{b} + 1
$$
Let $\tilde{t} = k + \lfloor \log{{k \choose b}} \rfloor - b$. Substituting $\tilde{t}$ in the inequality above,
$$
g_b(\tilde{t}, 2^k) \geq 2^{k - \tilde{t}} \cdot {\tilde{t} \choose b} - 2^{b} + 1 = 2^{b-\lfloor \log_2{{k \choose b} \rfloor}} \cdot {k + \lfloor \log{{k \choose b}} \rfloor - b \choose b} - 2^{b} + 1
$$
Then for $g_b(\tilde{t}, 2^k) > 1$ to hold, it is sufficient that 
$${k + \lfloor \log{{k \choose b}} \rfloor - b \choose b} > {k \choose b}, $$
which holds for $b \leq k/5$, so $t_b^* \geq \tilde{t} = k + \lfloor \log{{k \choose b}} \rfloor - b$. By Lemma (\ref{lem:expected}), the expected loss is at least $t_b^*/2 \geq 0.5 \left(k + \lfloor \log{{k \choose b}} \rfloor - b\right)$ as required. This completes the proof.
\end{proof}


\section{Abstract Theorem}

In this section we show that in fact any monotone function that satisfies the inequality of Lemma \ref{lem:general} and a simple boundary condition gives both an upper bound on the expected loss and a strategy that the forecaster can use to achieve an expected loss no higher than the value given by this function.

\begin{theorem}[Abstract Theorem] \label{thm:abstract}
	Let $f : \mathbb{N}^{b+1} \setminus \{(0, \ldots, 0) \} \rightarrow \Re$ be any function that satisfies the properties
	\begin{enumerate}
		\item Boundary: $f(0, \ldots, 0, 1) \geq 0$  
		\item Decomposition property: for every state $\vec{k}$, subset $A \subseteq [d]$ of $a \geq 1$ choices, and decomposition $(\vec{k}^1, \ldots, \vec{k}^d)$, 
		$$
		f(\vec{k}) \geq \frac{a-1}{a} + \frac{1}{a} \cdot \sum_{i \in A} f(\vec{s}_i)
		$$
		where $\vec{s}_i$ is the successor state if choice $i$ is correct.
	\end{enumerate}  
	
	Then the forecaster has a strategy ensuring the expected loss in the worst case is upper bounded by $f$ as follows: from any state $\vec{k}$, observe the decomposition given by the opinions of the experts and the successor states $\vec{s}_1, \ldots, \vec{s}_d$. Then predict choice $i$ with probability 
	$p_i = \max\{0, 1 + f(\vec{s}_i) - f(\vec{k})\}$ for all $i < d$ and choice $d$ with probability $p_d = 1 - \sum_{i < d} p_i$.
\end{theorem}

\smallskip

We note the decomposition property for $a=1$ choices is equivalent to monotonicity with respect to successors, i.e. $f(\vec{k}) \geq f(\vec{r})$ for any state $\vec{r}$ that is a successor of $\vec{k}$.

\smallskip

\begin{proof}[Proof of Theorem \ref{thm:abstract}]
	Denote the current state by $\vec{k}$ and the decomposition observed by $(\vec{k}^1, \ldots, \vec{k}^d)$.
	We first argue the strategy prescribed by the probabilities $p_i$ is valid, by showing that $p_i \in [0,1]$ for all $i$. First note that $f(\vec{s}_i) - f(\vec{k}) \leq 0$ since $\vec{s}_i$ is a successor of $\vec{k}$ and $f$ is monotonic. Thus for each $i < d$ we have $0 \leq \max\{0, 1 + f(\vec{s}_i) - f(\vec{k})\} \leq 1$. We also have that $p_d = 1 - \sum_{i < d} \leq 1$. It remains to show that $p_d \geq 0$, which is equivalent to	
	 $\sum_{i < d} p_i \leq 1$. Let $A = \{i < d \; | \; 1 + f(\vec{s}_i) - f(\vec{k}) > 0\}$ and denote $a = |A|$. Then 
	\begin{equation} \label{eq:pi_rewrite}
	\sum_{i < d} p_i = \sum_{i \in A} p_i = \sum_{i \in A} \left( 1 + f(\vec{s}_i) - f(\vec{k}) \right) = a - a \cdot f(\vec{k}) + \sum_{i \in A} f(\vec{s}_i)
	\end{equation}
	

By applying the decomposition property of $f$ we obtain that $\frac{a-1}{a} + \frac{1}{a} \cdot \sum_{i \in A} f(\vec{s}_i) \leq f(\vec{k})$, so the sum of probabilities can be bounded by
$$
\sum_{i < d} p_i =  a - a \cdot f(\vec{k}) + \sum_{i \in A} f(\vec{s}_i) \leq 1
$$
 Thus the probabilities $p_i$ are valid for each $i \in [d]$.

By induction it will follow that the expected loss is upper bounded by $f$. The base case is obtained when there is only one expert left that has made $b$ mistakes so far. Then this expert will never make another mistake, so $\ell(0, \ldots, 0, 1) = 0$ and by construction of $f$ we have $f(0, \ldots, 0) \geq 0 = \ell(0, \ldots, 0, 1)$.

Assume that $f$ is an upper bound on the expected loss when starting from any state $\ell$ reachable from $\vec{k}$. We show that it also holds for $\vec{k}$. By Lemma \ref{lem:minineq}, the adversary has an optimal strategy that consists of selecting some decomposition $(\vec{k}^1, \ldots, \vec{k}^d)$, a set of choices $A$ so that all the experts vote on a choice in $A$, and selecting the next choice from $A$ with probability $1/a$, where $a = |A|$. 
Let $\vec{s}_1, \ldots, \vec{s}_d$ be the successor states corresponding to each choice $i$ being correct. By the induction hypothesis, 
$\ell(\vec{s}_i) \leq f(\vec{s}_i)$ for each $i$. By applying the decomposition property of $f$ we can bound the loss at state $\vec{k}$ as follows
$$
\ell(\vec{k}) = \frac{a-1}{a} + \frac{1}{a} \cdot \sum_{i \in A} \ell(\vec{s}_i) \leq  \frac{a-1}{a} + \frac{1}{a} \cdot \sum_{i \in A} f(\vec{s}_i) \leq f(\vec{k})
$$
Thus $f$ is also an upper bound for the loss at $\vec{k}$ as required.
\end{proof}

\medskip

In the next proposition we will allow a state to contain real valued entries; the successor function on such states is defined in the same way as for states with integer entries.

\begin{proposition} \label{rmk:concave}
Let $f$ be a concave function such that for any state $\vec{k}$ and $a \in [d]$
	\begin{equation}\label{eq:simplerf}
	\frac{a-1}{a} + f(\vec{s}_1(\vec{k}_a)) \leq f(\vec{k}),
	\end{equation}
	where the vector $\vec{k}_a = (\vec{k}/a, \ldots, \vec{k}/a, 0, \ldots, 0)$ has exactly $d-a$ zero entries.

Then condition $3$ of Theorem \ref{thm:abstract} holds, i.e. for every subset $A \subseteq [d]$ of $a \geq 1$ choices and decomposition $(\vec{k}^1, \ldots, \vec{k}^d)$ we have
$$ \frac{a-1}{a} + \frac{1}{a} \cdot \sum_{i \in A} f\left(\vec{s}_i(\vec{k}^1, \ldots, \vec{k}^d)\right) \leq f(\vec{k})$$ 
\end{proposition}
\begin{proof}
	Let $\vec{k}$ be any state and $(\vec{k}^1, \ldots, \vec{k}^d)$ an arbitrary decomposition with successor states $\vec{t}_i = \vec{s}_i(\vec{k}^1, \ldots, \vec{k}^d)$, for 
	$i = 1 \ldots d$. Let $A \subseteq [d]$ be a set of $a \geq 1$ choices, which w.l.o.g. can be set to $A = \{1, \ldots, a\}$. By Lemma \ref{lem:major}, we get that if the decomposition $(\vec{k}^1, \ldots, \vec{k}^d)$ does have experts voting on choices outside $A$, then there is another decomposition $(\vec{h}^1, \ldots, \vec{h}^d)$ that dominates it with the property that $\vec{h}^i = \vec{0}$ for all $i \not \in A$. Then each successor $\vec{t}_i$ is reachable from the successor $\vec{r}_i = \vec{s}_i(\vec{h}^1, \ldots, \vec{h}^d)$ for each $i \in A$. By the monotonicity of $f$, it follows that $f(\vec{t}_i) \leq f(\vec{r}_i)$.
Thus we can assume $\vec{k}^i = \vec{0}$ for all $i \not \in A$. 

\medskip

	Together with the concavity of $f$, this implies 
\begin{align} \label{eq:fconcavecombo}
\frac{1}{a} \cdot \sum_{i \in A} f\left(\vec{s}_i(\vec{k}^1, \ldots, \vec{k}^d)\right) 
& = 
\frac{1}{a} \cdot \sum_{i \in A} f\left(\vec{t}_i\right) \notag \\
& \leq f\left(\sum_{i \in A}  \frac{\vec{t}_i}{a} \right) \notag \\
& =
	f(\vec{s}_1(\vec{k}_a))
\end{align}

Applying inequalities (\ref{eq:simplerf}) and (\ref{eq:fconcavecombo}), we obtain 
$$\frac{a-1}{a} + \frac{1}{a} \cdot \sum_{i \in A} 
f\left(\vec{s}_i(\vec{k}^1, \ldots, \vec{k}^d)\right) \leq \frac{a-1}{a} + f(\vec{s}_1(\vec{k}_a)) \leq f(\vec{k})$$ as required.
\end{proof}

\section{Algorithm and Upper Bound}


In this section we provide an upper bound and algorithm for approximately computing the function $\ell$.

\begin{lemma} \label{lem:log_combo}
For every $c > 1$, the function $f_{c} : \mathbb{N}^{b+1} \setminus \{(0, \ldots, 0)\} \rightarrow \Re$ defined as 
$$
f_{c}(k_0, \ldots, k_b) = \log_{\gamma}{\left(\sum_{i=0}^{b} c^{b-i} \cdot k_i\right)}
$$
satisfies the conditions of Theorem \ref{thm:abstract} for  $\gamma = \left(\frac{2c}{c+1} \right)^2$, and so upper bounds the minimax loss.
\end{lemma}
\begin{proof}
The first condition of Theorem \ref{thm:abstract} requires that $f_{c}(0, \ldots, 0, 1) \geq 0$, which trivially holds. The second condition is the decomposition property, which we show first for $a=1$ choices; this is equivalent to monotonicity with respect to successors. That is, we need $f_{c}(\vec{k}) \geq f_{c}(\vec{r})$ for any state $\vec{r}$ that is a successor of $\vec{k}$. It will suffice to show that $f_{c}(\vec{k}) \geq f_{c}(s_1(\vec{k}^1, \ldots, \vec{k}^d))$ for any decomposition $(\vec{k}^1, \ldots, \vec{k}^d)$ of $\vec{k}$, or equivalently,
\begin{small}
\begin{eqnarray*}
		&& c^b \cdot k_0 + c^{b-1} \cdot k_1 + \ldots + c^0 \cdot k_b \geq c^b \cdot k_0^1 + 
		c^{b-1} \cdot \left(k_1^1 + \sum_{v \neq 1} k_0^v\right) + \ldots + c^0 \cdot \left(k_{b}^1 + \sum_{v \neq 1} k_{b-1}^v\right) \iff \\
		& &
		c^b \cdot \left(k_0^1 + \ldots + k_0^d\right) + 
		\sum_{u=1}^b c^{b-u} \cdot \left(k_u^1 + \sum_{v \neq 1} k_{u}^v\right) \geq  
		 c^b \cdot k_0^1 + \sum_{u=1}^b c^{b-u} \cdot \left(k_{u}^1 + \sum_{v\neq 1} k_{u-1}^v\right)
\end{eqnarray*}
\end{small}
	Since $c > 1$, the coefficient of each $k_i^j$ on the left hand side is greater than or equal to its coefficient on the right hand side. Thus the second condition is also met. 
	
	To show the decomposition property for at least two choices, note first that the function $f_{c}$ is concave for any $c > 1$ since it is the composition of a log with a linear function. Let $a \in \{2, \ldots, d\}$. For any state $\vec{k}$, inequalities \ref{eq:equiv1}, \ref{eq:equiv2}, and \ref{eq:equiv_concave_fac} are equivalent:
	\begin{small}
		\begin{align} 
		& \frac{a-1}{a} + f_{c}\left(s_1\left(\vec{k}/a, \ldots, \vec{k}/a\right)\right) \leq f_{c}(\vec{k})  \label{eq:equiv1} \\
		& \frac{a-1}{a} + \log_{\gamma}{\left(c^{b} \cdot \frac{k_0}{a} + c^{b-1} \cdot \left( \frac{k_1}{a} + \frac{(a-1)k_0}{a} \right) + \ldots + \left( \frac{k_b}{a} + \frac{(a-1) k_{b-1}}{a}\right)\right)} \leq \log_{\gamma}{\left(c^{b} \cdot k_0 + \ldots + k_b\right)}  \label{eq:equiv2}  \\
		& \gamma^{\frac{a-1}{a}} \cdot c^b \cdot \frac{k_0}{a} + \sum_{u=1}^{b} \gamma^{\frac{a-1}{a}} \cdot c^{b-u} \cdot \left( \frac{k_u}{a} + \frac{(a-1)\cdot k_{u-1}}{a} \right) \leq c^b \cdot k_0 + c^{b-1} \cdot k_1 + \ldots + k_b \label{eq:equiv_concave_fac}
		\end{align}
	\end{small}

	Comparing the coefficients of $k_b$ in inequality \ref{eq:equiv_concave_fac} gives \begin{align} \label{eq:b}
	\gamma^{\frac{a-1}{a}} \cdot \frac{k_b}{a} \leq k_b \iff \gamma^{\frac{a-1}{a}} \leq a
	\end{align}
	while the coefficient of each $k_i$ for $i < b$ gives
	\begin{align} \label{eq:i}
	& \gamma^{\frac{a-1}{a}} \cdot \left[ c^{b-i} \cdot \frac{k_i}{a} + c^{b-i-1} \cdot \frac{(a-1)k_i}{a}\right] \leq c^{b-i} \cdot k_i \iff \notag \\
	& \gamma^{\frac{a-1}{a}} \cdot \left[\frac{c}{a} + \frac{a-1}{a}\right] \leq c \iff \notag \\
	& \left( \frac{2c}{c+1}\right)^{\frac{2a-2}{a}} \cdot \left[ 1 + \frac{c-1}{a} \right] \leq c
	\end{align}
	
	Note that \ref{eq:i} implies \ref{eq:b}, so it is sufficient to prove \ref{eq:i}. For $a=2$, inequality \ref{eq:i} is equivalent to $\left(\frac{2c}{c+1}\right) \cdot \frac{c+1}{2} \leq c$, which trivially holds. We show the function of $a$ representing the left hand side is decreasing in $a$ for $a \geq 2$, so it attains its maximum at $a=2$. We take $x = 1/a$ and obtain that \ref{eq:i} is equivalent to 
	\begin{align} \label{eq:essentialr}
	& \left( \frac{2c}{c+1}\right)^{2(1-x)} \cdot \left(1 + (c-1)x\right) \leq c \; \; \; \mbox{for} \; \; \;	 0 < x \leq \frac{1}{2}
	\end{align}
Define the function $g: (0, 2] \rightarrow \Re$ which is the log of the left hand side in the inequality above 
$$
g(x) = 2(1-x) \cdot \log{\left(\frac{2c}{c+1}\right)} + \log{\left(1+(c-1)x\right)}
$$ 
The derivative of $g$ is
$$
g'(x) = - 2 \log{ \left(\frac{2c}{c+1}\right)} + \frac{c-1}{1+(c-1)x}
$$
Using the inequality $\log{y} \leq y - 1$ we obtain $\log{\left(\frac{2c}{c+1}\right)} \leq \frac{c-1}{c+1}$, and so the derivative can be bounded by
$$
g'(x) \geq -2 \cdot \left( \frac{c-1}{c+1} \right) + \frac{c-1}{1 + \frac{c-1}{2}} = 0
$$
Thus inequality \ref{eq:essentialr} holds, which implies \ref{eq:i}. Thus the worst case is obtained for $a=2$, so inequality (\ref{eq:equiv_concave_fac}) holds for all $a \in \{2, \ldots, d\}$. Then  
	by Proposition \ref{rmk:concave}, the function $f_{c}$ satisfies condition (3) of the abstract theorem (Theorem \ref{thm:abstract}). It follows that $f_c$ satisfies all the conditions of Theorem \ref{thm:abstract} as required.
\end{proof}

\begin{figure}[h!]
	\caption{\textsc{Approximately Optimal Forecasting Algorithm}}
	\begin{tcolorbox}
		\emph{\textbf{Input:}} \begin{description}
			\item[$\clubsuit$] Set of $n$ experts, $d$ options, and upper bound $b$ on the number of mistakes that the best expert can make.
			\item[$\clubsuit$] Observe the predictions of the experts, which give a decomposition of the current state $\vec{k}$ into $(\vec{k}^1, \ldots, \vec{k}^d)$.
		\end{description}
		
		\emph{\textbf{Output:}} Probability vector $\vec{p} = (p_1, \ldots, p_d)$, where the forecaster selects choice $i$ with probability $p_i$.
		
		\bigskip
		
		\bigskip
		
		$\triangleright \; \;$ Compute 
		\begin{align}
		f(\vec{k}) = \min_{c > 1} \log_{\gamma}{(c^b \cdot k_0 + c^{b-1} \cdot k_1 + \ldots + k_b)}, \; \; \; \mbox{where} \; \; \; \gamma = \left( \frac{2c}{c+1} \right)^2 \notag
		\end{align}
		\emph{// For a fast approximation, taking $c = \log_{2}{(n)}/b$ suffices for the bound.}
		
		\medskip
		
		$\triangleright \; \; $ For every $i < d$, predict choice $i$ with probability 
		$
		p_i = \max\{ 0, 1 + f(\vec{s}_i) - f(\vec{k})\}
		$ 
		and $d$ with probability $p_d = 1 - \sum_{i < d} p_i$, where $\vec{s}_i$ is the successor state when choice $i$ is correct.
	\end{tcolorbox}
\end{figure}

\begin{theorem} \label{thm:ub}
	Consider a forecaster with $n$ experts, where the best expert makes at most $b= b(n)$ mistakes. Then
	the expected loss of the optimal forecaster is at most 
	$$\log_{4}{(n)} + b \cdot \left[\log_{4}{\left( \log_{2}{(n+1)} \right)} +4 \right] + 1.$$
\end{theorem}
\begin{proof}
	For every $c > 1$, let $\gamma = \left(\frac{2c}{c+1}\right)^2$ and define $f_{c} : \mathbb{N}^{b+1} \setminus \{(0, \ldots, 0)\} \rightarrow \Re$ as 
	$$
	f_{c}(k_0, \ldots, k_b) =  \log_{\gamma}{\left(\sum_{i=0}^{b} c^{b-i} \cdot k_i\right)}.
	$$
	By Lemma (\ref{lem:log_combo}), this function
	satisfies the conditions of Theorem \ref{thm:abstract}, so it gives an upper bound on the expected loss of an optimal forecaster in the worst case.

In order to get the required bound we will optimize the value of $c$ for every input state $\vec{k}$. 
 Observe that $\lim_{c \to 1} f_c(\vec{k}) = \infty$ and $\lim_{c \to \infty} f_c(\vec{k}) = \infty$, so the expression $\min_{c > 1} f_{c}(\vec{k})$ is well defined.
 Define $f : \mathbb{N}^{b+1} \setminus \{(0, \ldots, 0)\} \rightarrow \Re$, where
$$
f(\vec{k}) = \min_{c > 1} f_{c}(\vec{k})
$$

We show that $f$ 
gives the required upper bound. 
The function $f$ is the infimum of a family of concave functions, so it is also concave. Computing the value for $f(\vec{k})$ is equivalent to finding a global minimum of the function
$g : \Re \rightarrow \Re$ defined as 
$$
g(c) = \log_{\gamma}{\left( \sum_{j=0}^b c^{b-j} \cdot k_j \right)}, \; \; \mbox{where} \; \; \gamma = \left(\frac{2c}{c+1} \right)^2
$$
\smallskip

Take 
$$c = \frac{\log_{2}{(4n+4)}}{ \ln{(4)}}.$$
\smallskip

Using the inequality $\ln{x} \leq x - 1$, we can bound the expression $\log_{2}{\gamma}$ as follows:
\begin{align} \label{eq:boundgamma}
\log_{2}{\gamma} & = \log_{2}{\left[\left( \frac{2c}{c+1} \right)^2\right]} = 2 - 2\log_{2}{\left(\frac{c+1}{c} \right)} \notag \\
& \geq  2 - 2/\left(c \cdot \ln{2} \right)\notag \\
& =  \frac{2 (c \cdot \ln{2} - 1)}{c \cdot \ln{2}} = \frac{1}{0.5 + 1 / \log_{2}{(n+1)}}
\end{align}

Applying inequality (\ref{eq:boundgamma}) with this choice of $c$, the function $f_c$ evaluated at the start state $(n, 0, \ldots, 0)$ can be bounded by
\begin{align} \label{eq:ubn0star}
f_{c}(n, 0, \ldots, 0) & = \log_{\gamma}{(c^b \cdot n)} = \frac{\log_2{\left(c^b \cdot n\right)}}{\log_{2}{\gamma}} \notag \\
& \leq 
\left( \frac{1}{2} + \frac{1}{\log_{2}{(n+1)}} \right) \cdot \left( \log_{2}{(n)} + b \cdot \log_{2}{\left( 2 + \log_{2}{(n+1)}\right)}\right) \notag \\
& \leq \log_{4}{(n)} + 1 + b \cdot \left(\log_{4}{\left( \log_{2}{(n+1)} + 2\right)} + \frac{\log_{2}{\left( \log_{2}{(n+1)} + 2\right)}}{\log_{2}{(n+1)}} \right)  \notag \\ 
& \leq \log_{4}{(n)} + 1 + b \cdot \left(\log_{4}{\left( \log_{2}{(n+1)} \right)} +4 \right)  
\end{align}

Inequality (\ref{eq:ubn0star}) implies that the expected loss when starting with $n$ experts that have made no mistakes initially is bounded by
\begin{align}
\ell(n, 0, \ldots, 0) \leq f(n, 0, \ldots, 0) 
\leq  f_{c}(n, 0, \ldots, 0)  
 \leq  \log_{4}{(n)} + 1 + b \left(\log_{4}{\left( \log_{2}{(n+1)} \right)} +4 \right) \notag
\end{align}
as required.
\end{proof}

We can obtain in fact a more precise estimate when $b < \ln{(n)}/2$ by using a value of $c$ that depends on both $n$ and $b$.

\begin{theorem} \label{thm:ub_precise}
		Consider a forecaster with $n$ experts, where the best expert makes at most $b= b(n)$ mistakes. For any $1 \leq b < \ln{(n)}/2$,
	the expected loss of the optimal forecaster is at most 
	$$
	\left( 1 + 2b/\ln(n)\right) \left( \log_{4}{(n)} +
	b \cdot \log_{4}{\left(\log_{2}{n}/b\right)} \right).
	$$
\end{theorem}
\begin{proof}
The proof is similar to that of Theorem \ref{thm:ub}, except we define $
c = \log_{2}{(n)}/b.$ Recall we obtain the upper bound by evaluating the function $f_{c}(k_0, \ldots, k_b) =  \log_{\gamma}{\left(\sum_{i=0}^{b} c^{b-i} \cdot k_i\right)}$, where $\gamma = \left(\frac{2c}{c+1}\right)^2$, at the start state 
$(n, 0, \ldots, 0)$. Then 
\begin{align} \label{eq:idfc}
f_c(n, 0, \ldots, 0) = \log_{\gamma}{(n \cdot c^b)} = \log_{\gamma}{(n)} + b \cdot \log_{\gamma}{(c)} = \frac{\log_{4}{(n)} + b \cdot \log_{4}{\left(\frac{\log_{2}{n}}{b} \right)}}{\log_{2}{\left( \frac{2c}{c+1} \right)}}
\end{align}
For this choice of $c$, by using Taylor's inequality we obtain 
$$
\frac{1}{\log_{2}{\left(\frac{2c}{c+1} \right)}} \leq 1 + \frac{b}{\ln{n}} + \frac{2 b^2}{\ln{n}^2}
$$
Using the above inequality in identity \ref{eq:idfc} implies that
$$
f_c(n, 0, \ldots, 0) \leq \left(1 + \frac{2b}{\ln{n}}\right) \cdot \left(\log_{4}{(n)} + b \cdot \log_{4}{\left(\frac{\log_{2}{n}}{b} \right)} \right)
$$
which is the required bound.
\end{proof}
We note that the bound of Theorem \ref{thm:ub_precise} is within at most $O(b)$ (additive) of the lower bound.

\vspace{6mm}

\section{Figures}

On the next page we show next several plots, with the function $m$ that gives the exact expected loss, as well as the upper bounds given by the functions $f$ and $f_c$ (for a specific $c$) from Theorem \ref{thm:ub}. The plots in Figure (3.a) are for starting states of the form $(n,0)$, where $n = 0 \ldots 300$, while in Figure (3.b) for starting states $(n, 0, 0)$, where $n = 0 \ldots 80$.

\begin{figure}[h!]
	\centering
	\subfigure[For $b=1$ mistakes, the expected loss $\ell(n,0)$ for $n = 0 \ldots 300$ (in red), the upper bound given by $f(n,0)$ (in black) and the upper bound given by $f_c(n,0)$ with $c = \frac{\log_{2}{(4n+4)}}{2 \cdot \ln{2}}$ (in green).]
	{
		\includegraphics[scale = 0.55]{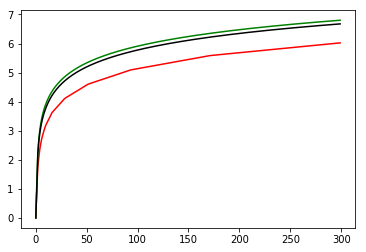}
	}
	\subfigure[For $b=2$ mistakes, the expected loss $\ell(n,0,0)$ for $n = 0 \ldots 80$ (in red), the upper bound given by $f(n,0,0)$ (in black) and the upper bound given by $f_c(n,0,0)$ with $c = \frac{\log_{2}{(4n+4)}}{2 \cdot \ln{2}}$ (in green).]
	{
		\includegraphics[scale = 0.55]{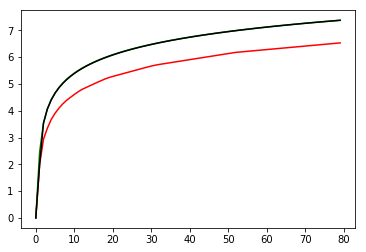}
	}
\caption{The loss function $\ell$ and its upper bounds.}
	\label{fig:upperbounds}
\end{figure}

\begin{figure}[h!]
	\centering
	\subfigure[The function $\ell(i,j)$.]
	{
		\centering
		\includegraphics[scale = 0.5]{m_ij_80.png}
		\label{fig:mij80}
	}\\
\subfigure[Upper bound given by the function $f_{c}(i,j)$,
$\;\;\;\;\;\;\;\;\;\;\;\;\;\;\;\;\;\;$where $c = \frac{\log_{2}{(4i+4j+4)}}{2 \cdot \ln{2}}.$]
{
	\includegraphics[scale = 0.5]{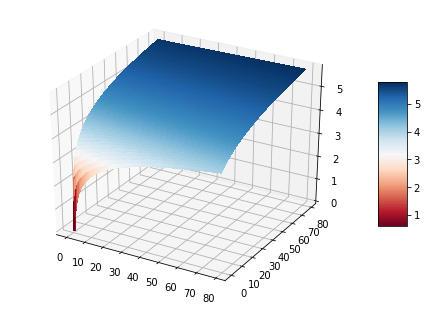}
	\label{fig:fapproxij80}
}
	\subfigure[Upper bound given by the function $\;\;\;\;\;\;\;\;\;\;\;\;\;\;f(i,j)= \inf_{c > 1} f_{c}(i,j)$.]
	{
		\includegraphics[scale = 0.5]{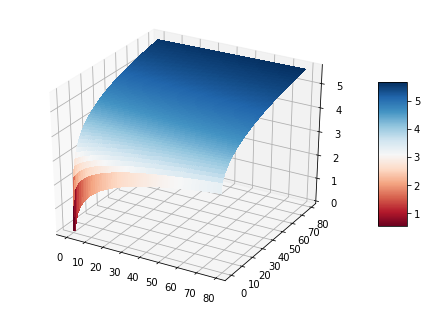}
		\label{fig:fij80}
	}
	\caption{For $b=1$ mistakes and $i, j = 0 \ldots 80$, the functions $\ell(i,j)$ (Figure a), $f(i,j)$ (Figure b), and $f_c(i,j)$ for $c= \frac{\log_{2}{(4i+4j+4)}}{2 \cdot \ln{2}}$ (Figure c).}
	\label{fig:mf}
\end{figure}

\newpage
\section{Acknowledgements}

We thank Jacob Abernethy for detailed discussion about his work on binary prediction, S\'ebastien Bubeck for pointers to the literature, and Nika Haghtalab for useful discussion in the early stages of the project.

\addcontentsline{toc}{section}{\protect\numberline{}References}%
\bibliography{expertbib}

\newpage
\appendix

\section{Values for the 2D problem}

In this section we include the exact values of the function $m(k,\ell)$ for $b=1$ and $k = 0 \ldots 39$, $\ell = 0 \ldots 15$.

\begin{figure}[h!]
\centering
\centerline{\includegraphics[scale=0.8]{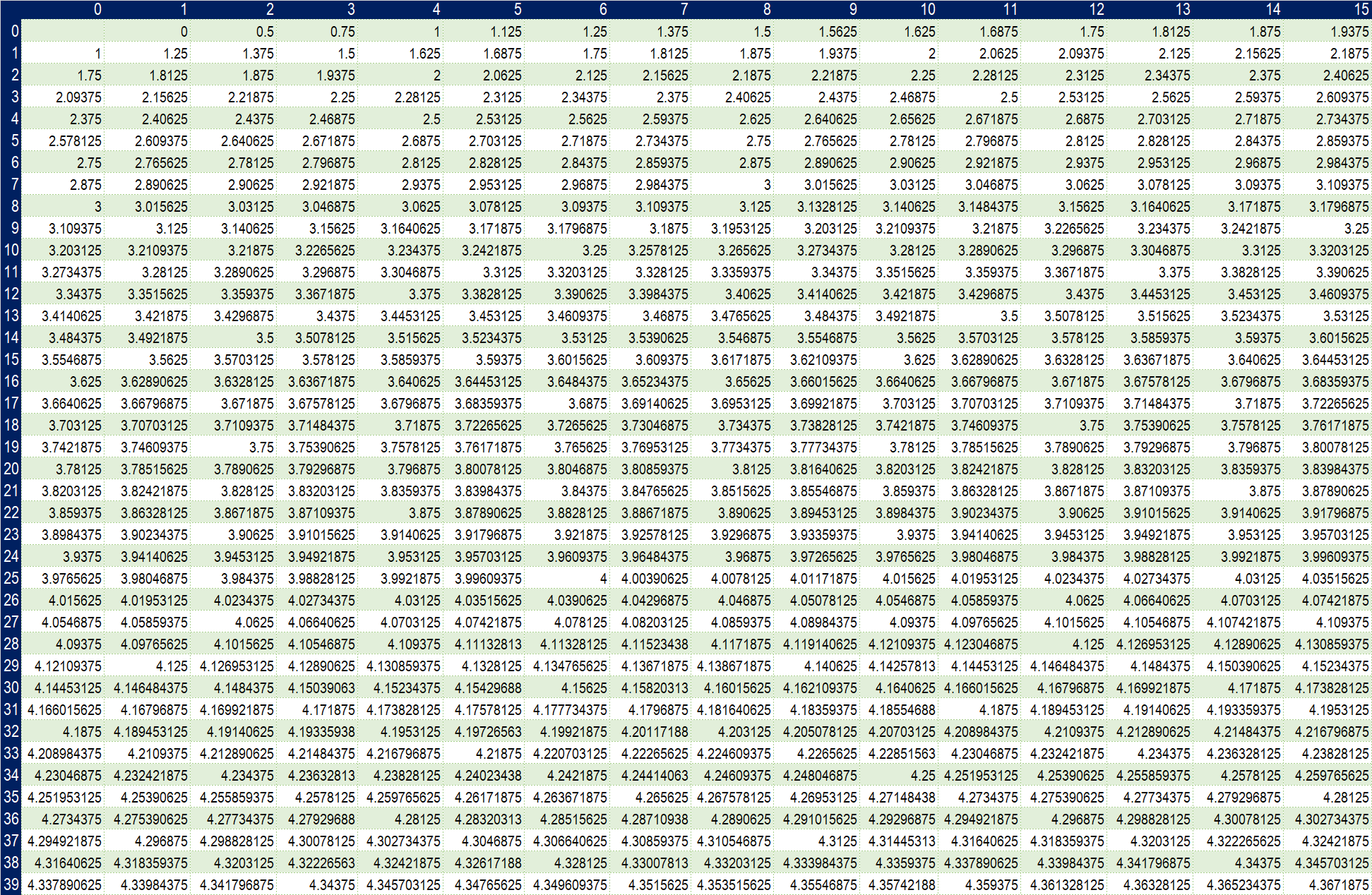}}
\label{fig:mistakes_ij}
\end{figure}

\end{document}